\documentclass[letterpaper]{article}

\usepackage{my_style}
\usepackage[margin=1in]{geometry}

\usepackage{microtype}
\usepackage{graphicx}
\usepackage{subfigure}
\usepackage{booktabs} 
\usepackage{hyperref}

\usepackage{my_style}

\usepackage{xcolor}
\usepackage{amsmath}
\usepackage{bm}
\usepackage{amssymb}
\usepackage{enumerate}
\usepackage{amsthm}
\ifx\theorem\undefined
\newtheorem{theorem}{Theorem}
\fi
\ifx\remark\undefined
\newtheorem{remark}[theorem]{Remark}
\fi
\ifx\assumption\undefined
\newtheorem{assumption}[theorem]{Assumption}
\fi
\ifx\definition\undefined
\newtheorem{definition}[theorem]{Definition}
\fi
\ifx\lemma\undefined
\newtheorem{lemma}[theorem]{Lemma}
\fi
\ifx\proposition\undefined
\newtheorem{proposition}[theorem]{Proposition}
\fi
\ifx\corollary\undefined
\newtheorem{corollary}[theorem]{Corollary}
\fi
\usepackage{xspace}
\def \AlgoName{\text{STORM-PG}\xspace}
\newcommand{\bxi}{\bm{\xi}}
\def\ep{\varepsilon}
\usepackage{natbib}
\usepackage{algorithm}
\usepackage{algorithmic}
\newcommand{\RR}{\mathbb{R}}
\newcommand{\EE}{\mathbb{E}}

\newcommand{\cB}{\mathcal{B}}
\newcommand{\cM}{\mathcal{M}}
\newcommand{\cS}{\mathcal{S}}
\newcommand{\cA}{\mathcal{A}}
\newcommand{\cP}{\mathcal{P}}
\newcommand{\cR}{\mathcal{R}}
\newcommand{\cF}{\mathcal{F}}
\newcommand{\cO}{\mathcal{O}}
\newcommand{\bg}{\mathbf{g}}

\def\blue#1{\textcolor{cyan}{#1}}

\def\red#1{\textcolor{black}{#1}}
\def\blue#1{}

\usepackage{titlesec}
\titlespacing\section{0pt}{11pt plus 2pt minus 2pt}{5pt plus 1pt minus 2pt}
\titlespacing\subsection{0pt}{2pt plus 2pt minus 2pt}{0pt plus 1pt minus 2pt}
\titlespacing\subsubsection{0pt}{2pt plus 2pt minus 2pt}{0pt plus 1pt minus 2pt}
\titlespacing\paragraph{0pt}{0pt plus 2pt minus 2pt}{5pt plus 2pt minus 2pt}
\usepackage{times}

\title{Stochastic Recursive Momentum for Policy Gradient Methods}
\author{%
Huizhuo Yuan
\\
Peking University
\\
\small{huizhuo.yuan@gmail.com}
\And
Xiangru Lian
\\
University of Rochester
\\
\small{admin@mail.xrlian.com}
\And
Ji Liu
\\
Kwai Inc.
\\
\small{ji.liu.uwisc@gmail.com}
\And
Yuren Zhou
\\
Duke University
\\
\small{yuren.zhou@duke.edu}
}

\def\discount{\alpha}

\begin{document}

\maketitle

\begin{abstract}
In this paper, we propose a novel algorithm named STOchastic Recursive Momentum for Policy Gradient (STORM-PG), which operates the stochastic recursive variance-reduced policy gradient method in an exponential moving average fashion. STORM-PG enjoys a provably sharp $O(1/\epsilon^3)$ sample complexity bound for STORM-PG, matching the best-known convergence rate for policy gradient algorithm. In the mean time, STORM-PG avoids the alternations between large batches and small batches which persists in comparable variance-reduced policy gradient methods, allowing considerably simpler parameter tuning. Numerical experiments depicts the superiority of our algorithm over comparative policy gradient algorithms.
\end{abstract}

\section{Introduction}

Reinforcement Learning (RL) \citep{Sutton:1998:IRL:551283} is a dynamic learning approach that interacts with the environment and execute actions according to the current state, so that a particular measure of cumulative rewards is maximized.
Model-free deep reinforcement learning algorithms \citep{lecun2015deep} have achieved remarkable performance in a range of challenging tasks, including stochastic control \citep{munos1998reinforcement}, autonomous driving \citep{shalev2016safe}, games \citep{mnih2013playing, silver2016mastering}, continuous robot control tasks \citep{Schulman:2015:TRP:3045118.3045319}, etc.

Generally, there are two aspects of methods of solving a model-free RL problem:
value-based methods such as Q-Learning~\citep{tesauro1995temporal}, SARSA~\citep{rummery1994line}, etc., as well as policy-based methods such as Policy Gradient (PG) algorithm \citep{Sutton:1999:PGM:3009657.3009806}.
PG algorithm models the state-to-action transition probabilities as a parameterized family, and the cumulative rewards can be regarded as a function of the parameters.
Thus, policy gradient based problem shares a formulation that is analogous to the traditional stochastic optimization problem.

One critical challenge of reinforcement learning algorithms compared to traditional gradient based algorithms lies on the issue of \textit{distribution shift}, that is, the data sample distribution encounters distributional changes throughout the learning dynamics \citep{pmlr-v80-papini18a}.
To correct this, (an off-policy version of) Policy Gradient (PG) \citep{Sutton:1999:PGM:3009657.3009806} method and Trust Region Policy Optimization (TRPO) \citep{Schulman:2015:TRP:3045118.3045319} method have been proposed as general off-policy algorithms to optimize policy parameters using gradient based methods.\footnote{%
In reinforcement learning literature, on-policy algorithms make use of samples rolled out by the current policy for only once, and hence suffer from high sample complexities.
On the contrary, off-policy algorithms in earlier work enjoy reduced sample complexities since they reuse the past trajectory samples~\citep{mnih2015human, lillicrap2015continuous}.
Nevertheless, they are often brittle and sensitive to hyperparameters and hence suffer from reproducibility issues \citep{henderson2018deep}.
}
PG method directly optimizes the policy parameters via gradient based algorithms, and it dates back to the introduction of REINFORCE~\citep{williams1992simple} and GPOMDP~\citep{baxter2001infinite} estimators that our algorithm is built upon.
%

The problem of high sample complexity arises frequently in policy gradient based methods due to a combined effect of high variance incurred during the training phase \citep{henderson2018deep, duan2016benchmarking} and distribution shift, limiting the ability of model-free deep reinforcement learning algorithms.
Such a combined effect signals the potential need of adopting variance-reduced gradient estimators \citep{johnson2013accelerating,nguyen2017sarah,zhou2018stochastic,fang2018spider} to accelerate off-policy algorithms.
Recently proposed variance-reduced policy gradient methods include SVRPG \citep{pmlr-v80-papini18a,xu2019improved} and SRVRPG \citep{xu2019sample}
theoretically improve the sample efficiency over PG.
This is corroborated by empirical findings: we observe that the variance-reduced gradient alternatives SVRPG and SRVRPG accelerate and stabilize the training processes, mainly due to their accommodations with larger stepsizes and reduced variances \citep{pmlr-v80-papini18a,DBLP:journals/corr/abs-1710-06034}.

Nevertheless compared to the vanilla PG method, one major drawback of the aforementioned variance-reduced policy gradient methods is their alternations between large and small batches of trajectory samples, spelled as the \textit{restarting mechanism}, so the variance can be effectively controlled.
In this paper, we circumvent such a restarting mechanism by introducing a new algorithm named STOchastic Recursive Momentum Policy Gradient (\AlgoName), which utilizes the idea of a recently proposed variance-reduced gradient method STORM \citep{cutkosky2019momentum} and blends with policy gradient methods.
STORM is an online variance-reduced gradient method that adopts an \textit{exponential moving averaging mechanism} that persistently discount the accumulated variance.
In the nonconvex smooth stochastic optimization setting, STORM achieves an $O(\epsilon^{-3})$ queries complexity that ties with online SARAH/SPIDER and matches the lower bound for finding an $\epsilon$-first-order stationary point \citep{arjevani2019lower}.
As a closely related variant, SARAH/SPIDER based stochastic variance-reduced compositional gradient methods also achieve an $O(\epsilon^{-3})$ complexity under a different set of assumptions \citep{hu2019efficient,zhang2019multi}.
Our proposed \AlgoName algorithm blends such a state-of-the-art variance-reduced gradient estimator with the PG algorithm.
Instead of introducing a restarting mechanism in concurrent variance-reduced policy gradient methods, our \AlgoName algorithm guarantees the variance stability by adopting the exponential moving averaging mechanism featured by STORM.
In our experiments, we see that the variance stability of our variance-reduced gradient estimator allows our \AlgoName algorithm to achieve a (perhaps surprisingly)  overall mean rewards improvement in reinforcement learning tasks.

\paragraph{Our Contributions}
We have designed a novel policy gradient method that enjoys several benign properties, such as using an exponential moving averaging mechanism instead of restarting mechanism to reduce our gradient estimator variance.
Theoretically, we prove a state-of-art convergence rate for our proposed \AlgoName algorithm in our setting.
Experimentally, our \AlgoName algorithm depicts strikingly desirable performance in many reinforcement learning tasks.

\paragraph{Notational Conventions}
Throughout the paper, we treat the parameters $L_d$, $C_\gamma$, $R$, $M$, $N$, $\Delta$ and $\sigma$ as global constants.
Let $h$ denote the index of steps that the agent takes to interact with the environment and $H$ is the maximum length of an episode.
Let $\|\cdot\|$ denote the Euclidean norm of a vector or the operator norm of a matrix induced by Euclidean norm. 
\blue{For fixed $t \ge 0$, let $\bxi_{0:t}$ denotes the sequence generated by the learning process $\{\bxi_0, \bxi_1, \bxi_2, \ldots, \bxi_t\}$.
Let $\cF_t$ denotes the $\sigma$-algebra generated by $\{\bxi_{0:t}\}$. $\EE[\cdot\mid\cF_t]$ is the conditional expectation up to time $t$.}
For fixed $t \ge 0$, let $\cB_t$ denotes the batch of samples choosen at the $t$'th iteration and $\cB_{0:t} = \{\cB_0, \cB_1, \ldots, \cB_t\}$. $\cF_t$ is the $\sigma$-algebra generated by $\cB_{0:t}$ and $\EE[\cdot\mid\cF_t]$ is the conditional expectation based on samples generated up to the $t$'th iteration.
Other notations are explained at their first appearances.

\paragraph{Organization}
The rest of our paper is organized as follows.
Section~\ref{sec:pre} introduces the backgrounds and preliminaries of the policy gradient algorithm.
Section~\ref{sec:storm} formally introduces our \AlgoName algorithm design.
Section~\ref{sec:def} introduces the necessary definitions and assumptions.
Section~\ref{sec:conv} presents the convergence rate analysis, whose corresponding proof is provided in Section~\ref{sec:proof}.
Section~\ref{sec:exp} conducts experimental comparison on continuous control tasks, and Section~\ref{sec:final} concludes our results.

\begin{table*}[t]
\centering
\begin{tabular}{lcr}
\hline
Algorithms & Complexity & Restarting \\
\hline
PGT \citep{sutton2000policy} & $O(\epsilon^{-4})$ &N\\
REINFORCE \citep{williams1992simple} & $O(\epsilon^{-4})$&N\\
GPOMDP \citep{baxter2001infinite} & $O(\epsilon^{-4})$&N\\
SVRPO \citep{DBLP:journals/corr/abs-1710-06034} & N/A&Y\\
SVRPG \citep{xu2019improved} & $O(\epsilon^{-10 / 3})$ &Y\\
SRVRPG \citep{xu2019sample} & $O(\epsilon^{-3})$&Y \\
STORM-PG (This paper) & $O(\epsilon^{-3})$ &N\\
\hline
\end{tabular}
\caption{Sample complexities of comparable algorithms for finding an $\epsilon$-accurate solution.}
\label{tab:my_label}
\end{table*}

\section{Policy Gradient Prelimilaries}
\label{sec:pre}
In this section we introduce the background of policy gradient and the objective function that our algorithm is based on. \red{The basic operation of the PG algorithm is similar to the gradient acsent algorithm with some RL specific gradient estimators. In Section~\ref{subsec:reinforce} we introduce the REINFORCE estimator which is the basis of many follow up PG works. In Section~\ref{subsec:gpomdp} we introduce the GPOMDP estimator which further reduces the variance and is the fundation of our algorithm. Finally in Section~\ref{subsec:gaussian} we formulate the probability induced by the policy as a Gaussian distribution, which is a special case adopted in our experiments.}
\subsection{REINFORCE Estimator}
\label{subsec:reinforce}
We consider the standard reinforcement learning setting of solving a discrete time finite horizon Markov Decision Process (MDP) $\cM = \{\cS, \cA, \cP, \cR, \gamma , \rho\}$ which models the behavior of an agent interacting with a given environment.
Let $\cS$ be the space of states in the environment, $\cA$ be the space of actions that the agent can take, $\cP:\cS \times \cA \rightarrow \cS$ be the transition probability from $s\in \cS$ to $s' \in \cS$ given $a\in \cA$, $R: \cS\times \cA \rightarrow \RR$ be the reward function of taking action $a\in \cA$ at state $s\in \cS$, $\gamma$ be the discount factor \red{that adds smaller weights to rewards at more distant future}, and $\rho$ be the initial state distribution.

We mainly focuses on in this paper the policy gradient setting where there is a \textit{policy} $\pi(a\mid s)$ as the probability of taking action $a$ given state $s$ such that $\sum_{a\in\cA} \pi(a\mid s) = 1$;
\blue{$\pi(\cdot\mid s)$ is called a \textit{policy} which models the agent's state transition after interacting with the environment.}
The \textit{policy} $\pi(\cdot\mid s)$ models the agent's behavior after experiencing the environment's state $s$.
Given finite state and action spaces, the policy $\pi(a\mid s)$ can be coded in a $|\cS| \times |\cA|$ tabular.
However when the state/action space is large or countably infinite, we adopt a probability mass function class $\pi_{\bxi}(a\mid s)$, parameterized by $\bxi \in \RR^d$, as an approximated class of functions to such a tabular.
Given a policy $\pi_{\bxi}(\cdot\mid s)$, the probability of a trajectory $\tau$ can be expressed in terms of the transition probability $p(s'\mid s,a)$ and the policy $\pi_{\bxi}(a\mid s)$: 
\begin{equation}\label{eq:defp}
p(\tau\mid \bxi)
=
\prod_{t=0}^{H-1} \pi_{\bxi}(a_t\mid s_t)\cdot p(s_{t+1}\mid s_t,a_t)
,
\end{equation}
where the trajectory $\tau := (s_0, a_0, s_1, a_1, \ldots, s_H, a_H)$ is the sequence that alters between states and actions, and $H$ is the maximum length (episode) of all trajectories.

Policy gradient algorithms target to maximize the expected sum of discounted rewards over trajectories $\tau$:
\begin{equation}
\label{eq:objective}
L(\bxi) \equiv \EE_{\tau \sim p(\cdot\mid \bxi)} R(\tau)
:= 
\EE_{\tau \sim p(\cdot\mid \bxi)} \left[
\sum_{t=0}^{H-1} \gamma^t r(s_t, a_t)
\right]
\end{equation}
where the expectation is taken over a parameterized probability distribution $p(\cdot \mid \bxi)$ with parameter $\bxi$, as is defined in~\eqref{eq:defp}.
Standard algorithm for maximizing \eqref{eq:objective} is the gradient descent algorithm (GD) which updates $\bxi_{t}$ on the direction of the objective gradient with a fixed learning rate $\eta$:
$$
\bxi_{t+1} = \bxi_t + \eta \nabla_{\bxi} L(\bxi_t),
$$
where the gradient $\nabla_{\bxi} L(\bxi)$ can be calculated as follows by combining~\eqref{eq:defp} and~\eqref{eq:objective}:
\begin{equation}
\begin{aligned}
&\quad
\nabla_{\bxi} L(\bxi)
\\&=
\nabla_{\bxi}\int p(\tau\mid\bxi)R(\tau)d\tau
= 
\int \nabla_{\bxi} p(\tau\mid\bxi) R(\tau) d\tau
\\&=
\int \frac{\nabla_{\bxi} p(\tau\mid\bxi)}{p(\tau\mid\bxi)}R(\tau)\ p(\tau\mid\bxi)d\tau
\\&=
\EE_{\tau\sim p(\cdot\mid\bxi)} \left[\nabla_{\bxi} \log p(\tau\mid\bxi)R(\tau)\right]
.
\end{aligned}
\end{equation}
To avoid the costly (or infeasible in the case of infinite spaces) full gradient $\nabla_{\bxi} L(\bxi)$ computations which requires sampling \textit{all possible trajectories}, we adopt its Monte Carlo estimator as:
\begin{equation}
\begin{aligned}
\label{eq:reinforce}
\hat{\nabla}_{\bxi} L(\bxi)
&=
\frac{1}{N} \sum_{i=1}^N \nabla \log p(\tau_i\mid\bxi) R(\tau_i)
,\end{aligned}
\end{equation}
where the trajectories $\tau_i$ are generated according to the trajectory distribution $p(\cdot\mid\bxi)$.
The above estimator in policy gradient is known as the REINFORCE estimator~\citep{williams1992simple}.

\subsection{GPOMDP Estimator}
\label{subsec:gpomdp}
One of the disadvantage of REINFORCE estimator lies on its excessive variance of trajectories introduced throughout the end of the episode.
Using a simple fact that for any constant $b$, $\EE [\nabla \log \pi_{\bxi}(a\mid s) b] = 0$ and the observation that rewards obtained before step $h$ is irrelevant with $\pi(a\mid s)$ after step $h$, the REINFORCE estimator~\eqref{eq:reinforce} can be substituted by the following GPOMDP~\citep{baxter2001infinite} unbiased estimator which uses a baseline to reduce the variance:
$$
\begin{aligned}
&\quad 
\hat{\nabla}_{\bxi} L(\bxi)
\\&\hspace{-.1in}
=
\frac{1}{N} \sum_{i=1}^N \sum_{h=0}^{H-1}\left(\sum_{t=0}^h\nabla \log \pi_{\bxi}(a_t\mid s_t)\right)
\left(\gamma^h r(s_h, a_h) -b_h\right)
,
\end{aligned}
$$
where for each $h\in [0,H-1]$, $b_h$ is a constant.
Throughout this paper, we use $d_i(\bxi)$ to refer to the unbiased GPOMDP estimator of $\nabla_{\bxi} L(\bxi)$:
\begin{equation}
\label{eq:estimator}
\begin{aligned}
&\quad
d_i(\bxi)
\\&=
\sum_{h=0}^{H-1}\left(\sum_{t=0}^h\nabla \log \pi_{\bxi}(a_t\mid s_t)\right)
\left(\gamma^h r(s_h, a_h) -b_h\right)
.
\end{aligned}
\end{equation}
where $(a_t, s_t)$ are action-state pairs along the trajectory $\tau_i$.
We adopt a variance-reduced version of GPOMDP estimator throughout the end of this paper.

\subsection{Gaussian Policy}
\label{subsec:gaussian}
Finally, we introduce the Gaussian policy setting.
In control tasks where the state and action spaces can be continuous, one choice of the policy function class is the Gaussian family:
$$
\pi_{\bxi}(a\mid s)
=
\frac{1}{\sqrt{2\pi\sigma^2}}\exp\left\{
-\frac{(\bxi^\top \psi(s) - a)^2}{2\sigma^2}
\right\}
,
$$
where $\sigma^2$ is the fixed variance parameter and $\psi(s):\cS \rightarrow \RR^d$ is a bounded feature mapping from the state space $\cS$ to $\RR^d$.
As the readers will see, the Gaussian policy satisfies all assumptions in Section~\ref{sec:def};
more detailed discussions can be found in~\citet{xu2019improved},~\citet{xu2019sample} and~\citet{pmlr-v80-papini18a}.

\section{STORM-PG Algorithm}
\label{sec:storm}
Recall our goal is to solve the general policy optimization problem:
\begin{equation}\label{eq:prob}
\text{maximize}_{\bxi} ~~
L(\bxi),
\end{equation}
and $d_i(\bxi)$ defined in \eqref{eq:estimator} is an unbiased estimator of the true gradient $\nabla_{\bxi} L(\bxi)$.
The simplest algorithm, stochastic gradient ascent, updates the iterates as
$$
\bxi_{t+1} = \bxi_t + \eta d_i(\bxi_t)
,
$$
where $i$ is chosen randomly from a data set sampled with the current distribution $\pi_{\bxi}$.
To further unfold this expression, we note that $d_i(\bxi) = \sum_{h=0}^{H-1} d_{i,h}(\bxi)$ where
$$
d_{i,h}(\bxi)
=
\left(\sum_{t=0}^h \nabla \log \pi_{\bxi}(a_t \mid s_t) \right)\left(\gamma^h r(s_h, a_h) - b_h\right)
.
$$
To remedy the distribution shift issue in reinforcement learning tasks, we introduce an importance sampling weight between trajectories generated by $\bxi$ and the ones generated by $\bxi '$ as
$$
d_i^{\bxi '}(\bxi) = \sum_{h=0}^{H-1}\frac{p(\tau_{i, h}\mid \bxi)}{p(\tau_{i, h} \mid \bxi ')} d_{i,h}(\bxi)
,
$$
where  $\tau_{i, h}$ is a trajectory generated by $p(\cdot \mid \bxi ')$ truncated at time $h$.
To further reduce the variance introduced by the randomness in $i$, SVRG introduced a variance-reduced estimator estimator of $\nabla_{\bxi} L(\bxi_t)$
\begin{equation}\label{eq:svrg}
\bg_t := d_i(\bxi_t) - d_i^{\bxi_{t+1}}(\tilde{\bxi}) + \tilde{u}
,
\end{equation} 
where $\tilde{\bxi}$ is a fixed point calculated once every $q$ steps and $\tilde{u}$ is a fixed estimation of the gradient at point $\tilde{\bxi}$.
Instead of the aforementioned SVRG-type estimator which was adopted by \citet{xu2019sample}, \citet{pmlr-v80-papini18a} adopts instead a recursive estimator 
\begin{equation}\label{eq:sarah}
\bg_{t+1} = d_i(\bxi_{t+1}) - d_i^{\bxi_{t+1}}(\bxi_{t}) + \bg_t
\end{equation} 
to track the gradient $\nabla_{\bxi} L(\bxi_{t+1})$ at each time. 
In above, $\bg_0$ is scheduled to be updated once every $q$ iterations as a large-batch estimated gradient.

\subsection{\AlgoName Estimator}
In this paper, we propose to use the STORM estimator as introduced in \citep{cutkosky2019momentum}, which is essentially an \textit{exponential moving average SARAH estimator} 
\begin{equation}
\label{eq:storm}
\bg_{t+1} = (1 - \discount)[d_i(\bxi_{t+1}) - d_i^{\bxi_{t+1}}(\bxi_t) + \bg_t] + \discount d_i(\bxi_{t+1})
.
\end{equation}
When $\alpha = 1$, the STORM-PG estimator reduces to the vanilla stochastic gradient estimator and when $\alpha = 0$, the STORM-PG esimator reduces to the SARAH estimator.
As our $\alpha$ is chosen between $(0, 1)$, the estimator is a combination of an variance reduced biased estimator and an unbiased estimator.
In addition, \eqref{eq:storm} can be rewritten as
$$
\bg_{t+1} = d_i(\bxi_{t+1}) + (1 - \discount)[\bg_t - d_i^{\bxi_{t+1}}(\bxi_t)]
,
$$
which can be interpreted as an exponentially decaying mechanism via a factor of $(1-\alpha)$.
We can see later in the proof of the convergence rate that the estimation error $\EE\|\bg_t - \nabla L(\bxi)\|^2$ can be controlled by a proper choice of $a$ while in SARAH case to control the convergence speed, the batch size $B$ or the learning rate $\eta$ have to be tuned accordingly.
This allows us to operate a single-loop algorithm instead of a double-loop algorithm.
We only need a large batch to estimate $\bg_0$ once, and do mini-batch or single batch updates till the end of the algorithm. 
This estimator hinders the accumulation of estimation error in each round.

We describe our \AlgoName as in Algorithm~\ref{alg:hybrid_sarah}.

\begin{algorithm}[t]
\caption{\AlgoName}
\label{alg:hybrid_sarah}
\begin{algorithmic}
\STATE \textbf{Input}: Number of epochs $T$, initial batch size $S_0$, step size $\eta$, mini-batch size $B$, initial parameter $\bxi_0$
\STATE Sample $S_0$ trajectories $\{\tau_i\}_{i\in \cS_0}$ from $p(\cdot\mid \bxi_0)$
\STATE
Calculate an initial estimate of $\nabla_{\bxi} L(\bxi_0)$:
\begin{equation}
\bg_{0} = \frac{1}{S_0} \sum_{i \in \cS_0} d_i(\bxi_0)
\end{equation}
\FOR {$t=0$ to $T-1$}
\STATE
Update $\bxi_{t+1}=\bxi_t + \eta \bg_t$
\STATE
Sample $B$ trajectories $\{\tau_i\}_{i\in\cB}$ from $p(\cdot\mid\bxi_{t+1})$
\begin{equation}
\label{eq:update}
\begin{aligned}
\bg_{t+1}
= &
(1 - \discount)\left(\frac{1}{B}\sum_{i\in\cB} \left[\bg_t - d_i^{\bxi_{t+1}}(\bxi_{t})\right]\right)
\\& +
 \frac{1}{B}\sum_{i\in\cB}d_i(\bxi_{t+1})\
\end{aligned}
\end{equation}

\ENDFOR
\STATE
Output $\widetilde{\bxi}$ chosen uniformly at random from $\{\bxi_t\}_{t=0}^{T-1}$
\end{algorithmic}
\end{algorithm}

\section{Definitions and Assumptions}
\label{sec:def}
In this section, we make several definitions and assumptions necessary for analyzing the convergence of the \AlgoName Algorithm. First of all, we define the $\epsilon$-accurate solution of a policy gradient algorithm:
\begin{definition}[$\epsilon$-accurate solution]
We call $\bxi \in \RR^d$ an $\epsilon$-accurate solution if and only if
$$
\|\nabla_{\bxi} L(\bxi)\|\le \epsilon
.
$$
\end{definition}
We say that an stochastic policy gradient based algorithm reaches an $\epsilon$-accurate solution if and only if
$$
\EE\|\nabla_{\bxi} L(\hat{\bxi})\|^2\le \epsilon^2
,
$$
where $\hat{\bxi}$ is the output after the algorithm's iteration number $T$, and the expectation is taken over the randomness in $\{\tau_i\}$ at each iteration.

To bound the norm of the gradient estimation $\|d_i(\bxi)\|$, we need assumptions on the norm of rewards $\|r(s, a)\|$ and the norm of gradient $\nabla_{\bxi} \log \pi_{\bxi}(a\mid s)$ as follows:

\begin{assumption}[Boundedness]
\label{assum:boundedness}
We assume that the reward and the gradient of $\log \pi_{\bxi}$ are bounded for any $a\in\cA$ and $s \in \cS$, and there exists a constant $R$ and a constant $M$ such that:
\begin{equation}
\|r(s, a)\| \le R, \qquad \|\nabla_{\bxi} \log \pi_{\bxi}(a\mid s)\|\le M
.
\end{equation}
\red{for any $a\in\cA, s \in \cS$.}
\end{assumption}

\begin{assumption}[Smoothness]
\label{assum:smoothness}
There exists a constant $N$ such that for any $a \in \cA$ and $s \in \cS$:
\begin{equation}
\|\nabla_{\bxi}^2 \log \pi_{\bxi}(a\mid s)\| \le N
.
\end{equation}
\end{assumption}

\begin{assumption}[Finite-variance]
\label{assum:finite_variance}
There exists a $\sigma \ge 0$ such that:
\begin{equation}
\text{Var}_{\tau_i\sim p(\cdot\mid\bxi)}(d_i(\bxi)) \le \sigma^2
.
\end{equation}
\end{assumption}

\begin{assumption}[Finite IS variance]
\label{assum:finite_IS_variance}
For $\bxi_1, \bxi_2 \in \RR^d$, use $w(\tau\mid\bxi_1, \bxi_2)$ to denote the importance sampling weight $p(\tau\mid\bxi_1)/p(\tau\mid\bxi_2)$. Then there exists a constant $\phi$ such that:
\begin{equation}
\text{Var}(w(\tau\mid\bxi_1, \bxi_2)) \le \phi^2
,
\end{equation}
where the variance is taken over $\tau \sim p(\cdot \mid \bxi_2)$.
\end{assumption}

\section{Convergence Analysis}
\label{sec:conv}
In this section, we introduce the lemmas neccessary for proving the convergence results of our \AlgoName Algorithm and finally state our main theorem of convergence. 
We recall that our goal is to achieve an $\epsilon$-accurate solution of function $L(\bxi)$, whose gradient can be estimated unbiasedly by $d_i(\bxi)$.
First of all, given Assumptions~\ref{assum:boundedness} and~\ref{assum:smoothness}, we can derive the boundedness, Liptchizness of $d_i(\bxi)$ and the smoothness of $L(\bxi)$, which are necessary conditions for proving convergence of nonconvex stochastic optimization problems.
From the definition in equation~\eqref{eq:estimator}, $d_i(\bxi)$ can be written as a linear combination of $\nabla_{\bxi} \log \pi_{\bxi}(a_h\mid s_h)$:
\begin{equation}
\label{eq:L}
d_i(\bxi) = \sum_{h=0}^{H-1} \left(\sum_{t=h}^{H-1} \gamma^t r(s_t, a_t)\right) \nabla_{\bxi} \log \pi_{\bxi}(a_h\mid s_h)
.
\end{equation}
Similarily, $\nabla_{\bxi} d_i(\bxi)$ can be written as a linear combination of $\nabla^2_{\bxi} \log \pi_{\bxi}(a_h\mid s_h)$.
Using the fact that $\sum_{h=0}^{H-1}\sum_{t=h}^{H-1} \gamma^t \le 1/(1-\gamma)^2$ and the bound derived in Assumptions~\ref{assum:boundedness} and~\ref{assum:smoothness}, it is direct to see that
\begin{equation}
\label{eq:20}
\|d_i(\bxi)\| \le \frac{MR}{(1-\gamma)^2},\qquad \|\nabla_{\bxi} d_i(\bxi)\| \le \frac{NR}{(1-\gamma)^2}
.
\end{equation}
Equation~\eqref{eq:20} implies that if we define $L_d = \frac{NR}{(1-\gamma)^2}$, $\|d_i(\bxi_1) - d_i(\bxi_2)\| \le L_d \|\bxi_1 - \bxi_2\|$ and $L(\bxi)$ is $L_d$-smooth.
With the boundedness and smoothness results, we further estimate the accumulated estimation error $\sum_{t=0}^{T-1}\EE\|\bg_t - \nabla_{\bxi} L(\bxi_t)\|^2$. 
In Lemma~\ref{lem:bounded_IS_variance} below we establish the variance bound of the importance sampling weight:
\begin{lemma}[Lemma A.1 in \citep{xu2019sample}]
\label{lem:bounded_IS_variance}
Let Assumptions~\ref{assum:boundedness},~\ref{assum:smoothness} and~\ref{assum:finite_IS_variance} hold.
Use $w_h(\tau\mid\bxi_1, \bxi_2)$ to denote the importance sampling weight $p(\tau_h\mid\bxi_1)/p(\tau_h\mid\bxi_2)$. Then there exists a constant $C = h(2hM^2 + N)(\phi + 1)$ such that:
\begin{equation}
\text{Var}(w_h(\tau\mid\bxi_1, \bxi_2)) \le C\|\bxi_1 - \bxi_2\|^2
,
\end{equation}
where the trajectory $\tau_{h}$ is the trajectory generated following the distribution $p(\cdot \mid \bxi_2)$ and truncated up to time $h$.
The variance is taken over $\tau \sim p(\cdot \mid \bxi_2)$.
\end{lemma}
The proof of Lemma~\ref{lem:bounded_IS_variance} can be found in \citep{xu2019sample}.
Combining Lemma~\ref{lem:bounded_IS_variance} and Equation~\eqref{eq:L}, we get the following bound of difference between two consecutive estimations:

\begin{lemma}
\label{lem:diff}
$$
\EE \|d_i^{\bxi_{t+1}}(\bxi_t) - d_i(\bxi_{t+1})\|^2 \le C_\gamma \|\bxi_{t+1} - \bxi_t\|^2
,
$$
where $C_\gamma$ is a constant depending on $\gamma$.
\end{lemma}

\red{Lemma~\ref{lem:diff} shows that the expected squared error between $d_i^{\bxi_{t+1}}(\bxi_t)$ and $d_i(\bxi_{t+1})$ is bounded by the squared distance between $\bxi_{t}$ and $\bxi_{t+1}$ by a constant dependent of $\gamma$ while independent of $H$. The specific choice of $C_{\gamma}$ and the proof of Lemma~\ref{lem:diff} can be found in Appendix~\ref{sec:append:lem:diff}.}

To estimate the estimation error $\EE\|\bg_t - \nabla_{\bxi} L(\bxi_t)\|^2$, we recursively calculate the relation between $\EE\|\bg_{t+1} - \nabla_{\bxi} L(\bxi_{t+1})\|^2$ and $\EE\|\bg_t - \nabla_{\bxi} L(\bxi_t)\|^2$ by bringing in the recursive definition of $\bg_{t+1}$ in Equation~\eqref{eq:storm}. The result is shown in Lemma~\ref{lem:one_recursive} below:

\begin{lemma}
\label{lem:one_recursive}
Let Assumption~\ref{assum:boundedness},~\ref{assum:smoothness},~\ref{assum:finite_variance} and~\ref{assum:finite_IS_variance} hold. Suppose that $\bg_t$ and $\bxi_t$ are the iteration sequence as defined in Algorithm~\ref{alg:hybrid_sarah} at time $t$. $L(\bxi)$ is the objective function to be optimized. Then the estimation error can be bounded by
\begin{equation}
\label{eq:one_recursive}
\begin{aligned}
& \quad
\EE \|\bg_{t+1} - \nabla_{\bxi} L(\bxi_{t+1})\|^2 
\\&\le
(1 - \discount)^2  \EE\|\bg_t - \nabla_{\bxi} L(\bxi)\|^2 
\\&+ 
\frac{2\eta^2}{B}(1 - \discount)^2 C_\gamma^2  \EE\|\bg_t\|^2+ \frac{2 \discount^2 \sigma^2}{B}
.
\end{aligned}
\end{equation}
\end{lemma}
The above lemma shows that the estimation error between $\bg_{t+1}$ and $\nabla_{\bxi} L(\bxi_{t+1})$ can be bounded by $(1 - \discount)^2$ times the estimation error of the previous iteration $\bg_t - \nabla_{\bxi} L(\bxi_t)$ plus a factor of the norm of $\|\bg_t\|^2$ plus a variance controlling term.


Lemma~\ref{lem:recursive} follows Lemma~\ref{lem:one_recursive} and is the main ingredients of proving the main theorem:
\begin{lemma}
\label{lem:recursive}
Let Assumption~\ref{assum:boundedness},~\ref{assum:smoothness},~\ref{assum:finite_variance} and~\ref{assum:finite_IS_variance} hold. Then the accumulated sum of expected estimation error $\sum_{t=0}^{T-1}\EE\|\bg_t - \nabla_{\bxi} L(\bxi_t)\|^2$ satisfies the following inequality:
\begin{equation}
\label{eq:recursive}
\begin{aligned}
&\quad \sum_{t=0}^{T-1} \EE\|\bg_t - \nabla_{\bxi} L(\bxi_t)\|^2
\\&\le
\frac{2}{\discount}\left[\frac{C_\gamma^2 \eta^2}{B} \sum_{t=0}^{T-1} \EE\|\bg_t\|^2
+
\frac{T \discount^2 \sigma^2}{B}
\right.
\\&+\left.
\EE\left[\|\bg_0 - \nabla_{\bxi} L(\bxi_0)\|^2\right]\right]
.
\end{aligned}
\end{equation}
\end{lemma}

\begin{remark}
We notice that in the proof of SARAH algorithm \citep{nguyen2017sarah} we have:
\begin{equation}
\label{eq:sarah_error}
\begin{aligned}
\EE\|\bg_{t+1} - \nabla_{\bxi} L(\bxi_{t+1})\|^2 
&\le \EE\|\bg_t  -\nabla_{\bxi} L(\bxi_t)\|^2
\\&\quad+ L^2\eta^2\EE\|\bg_{t}\|^2
,
\end{aligned}
\end{equation}
and
\begin{equation}
\label{eq:sarah_recursive}
\begin{aligned}
\sum_{t=0}^{T-1}\EE\|\bg_t - \nabla_{\bxi} L(\bxi_t)\|^2 
&\le L^2\eta^2\sum_{t=0}^{T-1}\sum_{s=1}^t\EE\|\bg_{s-1}\|^2 
\\&\le
 L^2\eta^2 T\sum_{t=0}^{T-1}\EE\|\bg_{t}\|^2
.
\end{aligned}
\end{equation}

Hence, to control the growth of function value, $\eta$ should be chosen with an order of $\cO(T^{-1/2})$. With infinitely increasing $T$, $\eta$ have to be chosen to be infinitely small. SARAH/SPIDER algorithm uses an restart machenism to remedy for this problem.
However in our \AlgoName Algorithm, by introducing a exponential moving average, we bring in a shrinkage term $(1 - \discount)^2$ on the accumulation speed of $\EE\|\bg_t\|^2$, allowing the order of $\sum_{t=0}^{T-1}\EE\|\bg_t - \nabla_{\bxi} L(\bxi_t)\|^2$ to decrease from $T$ to $\frac{1}{\discount}$.

For $\discount$, we only need to control $\discount \ge 4C_\gamma^2\eta^2$ so that $\eta$ is no longer related with $T$. This allows us to do continuous training without restarting the iterations.
\end{remark}

Next we come to our main theorem in this paper, which conclude that after $T$ iterations, the expected gradient norm satisfies a bound described below:

\begin{theorem}
\label{thm:main}
Let Assumptions~\ref{assum:boundedness}, ~\ref{assum:smoothness},~\ref{assum:finite_variance}and~\ref{assum:finite_IS_variance} hold.
When $\discount B \ge 4\eta^2 L_d^2$, the resulting point after $T$ iterates satisfies:
\begin{equation}
\label{eq:main}
\frac{1}{T} \sum_{t=0}^{T-1}\EE\|\nabla_{\bxi} L(\bxi_t)\|^2
 \le 
\frac{2\Delta}{\eta T} + \frac{2\discount \sigma^2}{B} + \frac{2}{T}\cdot \frac{\sigma^2}{S_0 \discount}
,
\end{equation}
where $\Delta := L(\bxi_0) - f^*$ is a constant representing the function value gap between the initialization and the optimal value $f^*$.
\end{theorem}

Choose $S_0 = \cO(\sigma^2\epsilon^{-2})$ and $B = \cO(\sigma^2\epsilon^{-1})$
In the theorem, the $\discount / B$ term can be controlled by letting $a$ to be proportional with $B/S_0$. Thus the third term is of order $\mathcal{O}(\frac{1}{T B})$ and the second term is of order $\mathcal{O}(\frac{1}{S_0})$.
If we choose $S_0 \discount = B$ and $\eta$ is of order $\cO(1)$, We have that after $T$ iterates, the algorithm reaches a point with expected gradient norm of order $\cO(\frac{1}{T} + \frac{\sigma^2}{S_0} + \frac{\sigma^2}{TB})$. Compared with $\cO(\frac{1}{T} + \frac{\sigma^2}{S_0} + \frac{1}{B})$ in \citep{pmlr-v80-papini18a} and $\cO(\frac{1}{T} + \frac{\sigma^2}{S_0})$ in \citet{xu2019sample}.
However, The sample complexity in \citet{xu2019sample} is $\sqrt{T}S_0 + TB$ while in our algorithm is $S_0 + TB$, which makes the algorithm converges faster.

The detailed analysis of the convergence rate is shown in the next section. 
Corollary~\ref{corr:main} is a direct result after Theorem~\ref{thm:main}. By controlling the estimated gradient to be in the $\epsilon$-neighborhood of 0, and minimizing $S_0$, we get the IFO complexity bound of \AlgoName algorithm:
\begin{corollary}
\label{corr:main}
Let Assumptions~\ref{assum:boundedness}, ~\ref{assum:smoothness},~\ref{assum:finite_variance} and~\ref{assum:finite_IS_variance} hold. Choose $\eta = \frac{\varepsilon}{2\sqrt{6}\sigma L_d}$, $\discount = \frac{\varepsilon^2}{6\sigma^2}$, and $S_0 = \frac{2\sqrt{3}}{3}\sigma^2\varepsilon^{-2}$. 
The IFO complexity of achieving an $\epsilon$-accurate solution is $\cO(\Delta L_d \cdot \frac{\sigma}{\epsilon^3} + \frac{\sigma^2}{\epsilon^2})$.
\end{corollary}

\begin{remark}
In the case of Gaussian policy \citep{xu2019sample}, one can also obtain an $O((1-\gamma)^{-4}\ep^{-3})$ trajectories sample upper bound to output an $\tilde{\bxi}$ such that $\EE\|\nabla_{\bxi} L(\tilde{\bxi}) \|^2 \le \epsilon^2$. 
We omit the detailed discussions and refer the reader to \citep{xu2019sample} for more details.
\end{remark}

\section{Proof of Main Results}
\label{sec:proof}
In this section, we prove the main results in this paper.
More auxiliary proofs are located in the supplementary section.

\subsection{Proof of Theorem~\ref{thm:main}}
\begin{proof}[Proof of Theorem~\ref{thm:main}]
By applying the $L_d$-smoothness of $L(\bxi)$ \red{\eqref{eq:20}}, we get a general estimation bound of $L(\bxi_{t+1})$:
\begin{equation}\begin{aligned}
\label{spiderbdd}
L(\bxi_{t+1}) &= L(\bxi_t + \eta \bg_t)
\\&\ge
L(\bxi_t) + \eta\langle \bg_t, \nabla_{\bxi} L(\bxi_t)\rangle - \frac{\eta^2 L_d}{2}\|\bg_t\|^2
\\&\overset{(a)}{\ge}
L(\bxi_t) + \left(\frac{\eta}{2} - \frac{\eta^2 L_d}{2}\right) \|\bg_t\|^2
\\&\quad
+ \frac{\eta}{2}\|\nabla_{\bxi} L(\bxi_t)\|^2 - \frac{\eta}{2}\|\bg_t - \nabla_{\bxi} L(\bxi_t)\|^2
\\&\overset{(b)}{\ge}
L(\bxi_t) + \frac{\eta}{4}\|\bg_t\|^2 + \frac{\eta}{2}\|\nabla_{\bxi} L(\bxi_t)\|^2
\\&\quad
- \frac{\eta}{2}\|\bg_t - \nabla_{\bxi} L(\bxi_t)\|^2
.
\end{aligned}
\end{equation}
In $(a)$ we apply the properties of inner product:
\begin{equation}
\begin{aligned}
&\quad \langle \bg_t, \nabla_{\bxi} L(\bxi_t)\rangle
\\&=
\frac{\|\bg_t\|^2}{2} + \frac{\|\nabla_{\bxi} L(\bxi_t)\|^2}{2} - \frac{\|\bg_t - \nabla_{\bxi} L(\bxi_t)\|^2}{2}
,
\end{aligned}
\end{equation}
and in $(b)$ we set $L_d \eta \le 1/2$.

Summing $L(\bxi_{t+1}) - L(\bxi_t)$ over $T$. We result in the inequality below:

\begin{equation}\label{eq:onestep}
\begin{split}
L(\bxi_T) - L(\bxi_0)
 &\le 
- \dfrac{\eta}{2} \displaystyle\sum_{t=0}^{T-1}\|\nabla_{\bxi} L(\bxi_t)\|^2 
- \frac{\eta}{4} \sum_{t=0}^{T-1}\|\bg_t\|^2 
\\&\quad
+ \frac{\eta}{2} \sum_{t=0}^{T-1}\|\bg_t - \nabla_{\bxi} L(\bxi_t)\|^2
.
\end{split}
\end{equation}

Since the LHS of \eqref{eq:onestep} is $\ge -\Delta$, taking its expectation along with \eqref{eq:recursive} in Lemma~\ref{lem:recursive} gives
$$\begin{aligned}
-\Delta
 &\le 
- \frac{\eta_t}{2} \sum_{t=0}^{T-1}\|\nabla_{\bxi} L(\bxi)\|^2 
-\frac{\eta}{4} \sum_{t=0}^{T- 1}\EE\|\bg_t\|^2 
\\&\quad
+ \frac{\eta}{2}\sum_{t=0}^{T-1}\EE\|\bg_t - \nabla_{\bxi} L(\bxi_t)\|^2
 \\&\le
- \frac{\eta}{2} \sum_{t=0}^{T-1}\|\nabla_{\bxi} L(\bxi_t)\|^2 
-\frac{\eta}{4} \sum_{t=0}^{T- 1}\EE\|\bg_t\|^2 
\\&\quad
+ \eta\cdot\frac{C_\gamma^2\eta^2}{\discount B} \sum_{t=0}^{T-1} \EE\|\bg_t\|^2
+\cdot \frac{T \discount \sigma^2\eta}{B}
+ \cdot \frac{\sigma^2\eta}{S_0 \discount}
\\&=
- \frac{\eta}{2} \sum_{t=0}^{T-1}\|\nabla_{\bxi} L(\bxi_t)\|^2 
\\&
- \frac{\eta}{4} \left( 1 - \frac{4C_\gamma^2\eta^2}{\discount B}\right) \sum_{t=0}^{T- 1}\EE\|\bg_t\|^2
\\&\quad
+\frac{ \eta\cdot T \discount \sigma^2}{B}
+ \eta\cdot \frac{\sigma^2}{S_0 \discount}
.
\end{aligned}$$

Our pick of $\eta$ satisfies $4C_\gamma^2\eta^2\le \discount B$ so $1 - \frac{4C_\gamma^2\eta^2}{\discount B} \ge 0$ and hence
\begin{equation}
\label{eq:mid}
\frac{\eta}{2} \sum_{t=0}^{T- 1}\EE\|\nabla_{\bxi} L(\bxi_t)\|^2
 \le 
\Delta + \frac{\eta\cdot T \discount \sigma^2}{B} + \eta\cdot \frac{\sigma^2}{S_0 \discount}
\end{equation}

Multiply both sides of Equation~\eqref{eq:mid} by $\frac{2}{\eta T}$:
$$
\frac{1}{T} \sum_{t=0}^{T-1}\EE\|\nabla_{\bxi} L(\bxi_t)\|^2
 \le 
\frac{2\Delta}{\eta T} + \frac{2\discount \sigma^2}{B} + \frac{2}{T}\cdot \frac{\sigma^2}{S_0 \discount}
,
$$
which completes our proof.
\end{proof}

\subsection{Proof of Corollary~\ref{corr:main}}
\begin{proof}[Proof of Corollary~\ref{corr:main}]

For choosing parameters of correct dependency over $\ep$, by Equation~\eqref{eq:main}, one requires:
\begin{equation}
\label{eq:limits}
2\discount\sigma^2 \le \frac{\ep^2 \cdot B}{3}
\qquad
\frac{2\Delta}{\eta T} \le \frac{\ep^2}{3}
\qquad
\frac{2}{T}\cdot \frac{\sigma^2}{S_0 \discount} \le \frac{\ep^2}{3}
\end{equation}
and we recall that previously we have a lower bound on $\discount$:
$
4C_\gamma^2\eta^2 \le \discount B
$.
So finally we choose
\begin{equation}
\label{eq:choose}
\discount = \frac{\ep^2 B}{6\sigma^2},
\qquad
\eta = \sqrt{\frac{\discount B}{4C_\gamma^2}}
= \frac{\ep B}{2\sqrt{6}\sigma C_\gamma}
.
\end{equation}

Bring Equation~\eqref{eq:choose} into Equation~\eqref{eq:limits} we have two lower bounds over $T$ to reach an $\epsilon$-accurate solution:
$$
T 
\ge 
\frac{2\Delta}{3\eta \ep^2} 
= 
\frac{2\Delta}{\ep^2}\cdot \frac{2\sqrt{6}\sigma C_\gamma}{3\ep B} 
= 
4 \sqrt{6}\Delta C_\gamma \cdot\frac{\sigma}{3\ep^3 B}
$$
and also
$$
T\ge 
\frac{2}{\ep^2}\cdot \frac{\sigma^2}{3S_0 \discount}
=
\frac{4\sigma^4}{3S_0 \ep^4 B}
.
$$

Our goal is to minimize the IFO complexity

\begin{align*}
&
\text{minimize}_{S_0}~ \text{IFO} = S_0 + TB
\\&s.t.~
T\ge  \max\left(4 \Delta L_d\cdot\frac{\sigma}{3\ep^3}, 4\frac{\sigma^4}{3S_0 \ep^4} \right)
\end{align*}
which is approximately equivalent to
$$
\text{minimize}_{S_0}~ S_0 + \frac{4\sigma^4}{3S_0 \ep^4}  + 4 \Delta L_d\cdot\frac{\sigma}{3\ep^3}
$$
Our best choice of $S_0$ is obviously $S_0 = 2\sigma^2 \ep^{-2}/\sqrt{3}$.
So the IFO complexity of reaching an $\epsilon$-accurate solution is
$$
\begin{aligned}
\text{IFO} 
&= 
\frac{2\sigma^2}{\sqrt{3}\ep^2} + \max\left(
4 \Delta L_d\cdot\frac{\sigma}{\ep^3}, \frac{2\sigma^2}{\sqrt{3}\ep^2} 
\right)
\\&\le
\frac{2\sigma^2}{\sqrt{3}\ep^2} + \left(
4 \Delta L_d\cdot\frac{\sigma}{\ep^3} + \frac{2\sigma^2}{\sqrt{3}\ep^2} 
\right)
\\&=
4 \Delta L_d\cdot\frac{\sigma}{\ep^3} + \frac{4\sigma^2}{\sqrt{3}\ep^2} 
.
\end{aligned}
$$
\end{proof}

\section{Experiments}
\label{sec:exp}

\begin{figure}[!tb]
\centering
\includegraphics[scale=0.5]{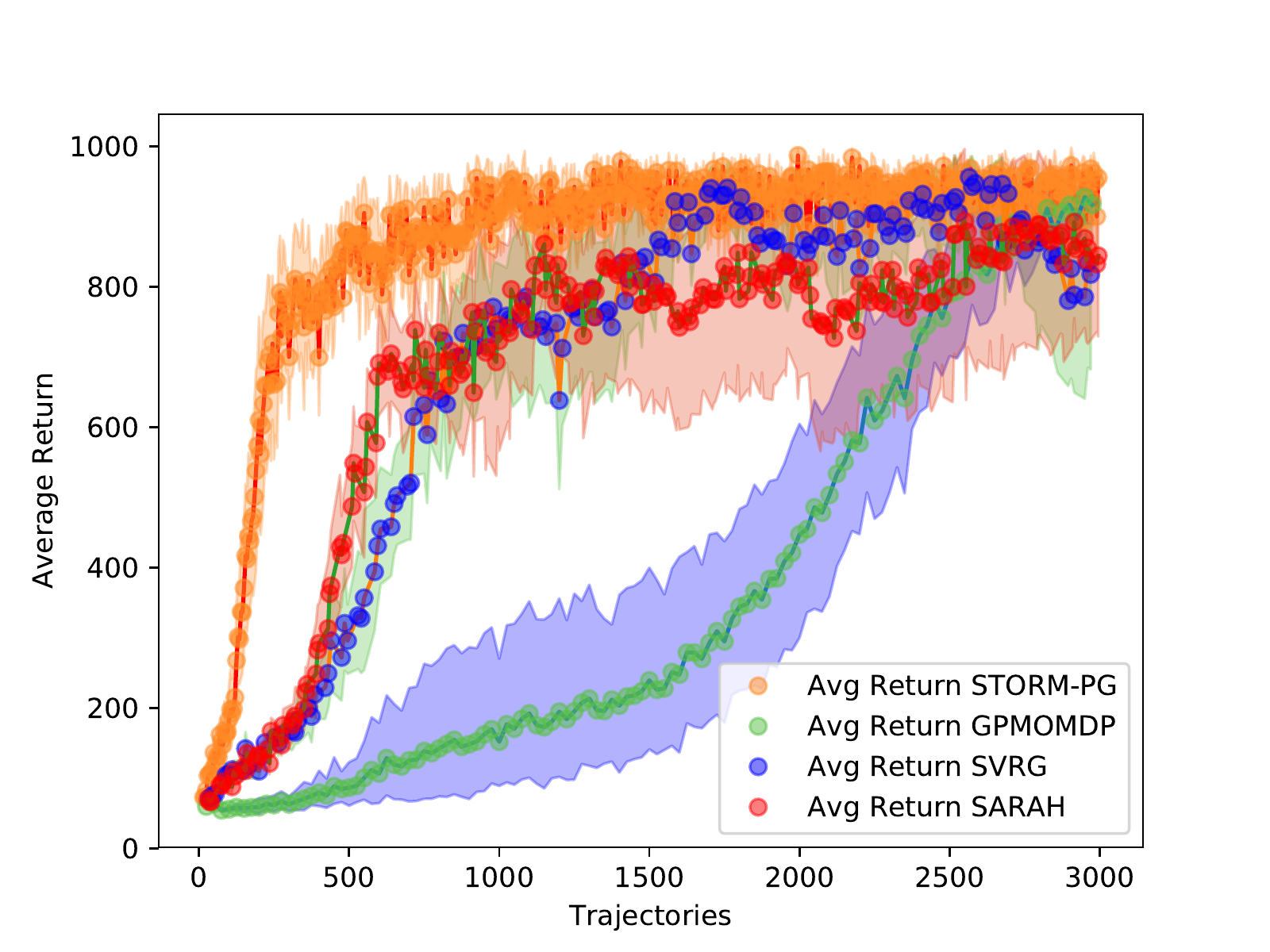}
\caption{A comparison between different policy gradient algorithms on Cart-Pole task. The $x$-axis is the trajectories sampled, the $y$-axis is the average return of the policy parameter.}
\label{fig:label}
\end{figure}

In this section, we design a set of experiments to validate the superiority of our STORM-PG Algorithm. Our implementation is based on the rllab library\footnote{\url{https://github.com/Dam930/rllab}} and the initial implementation of~\citet{pmlr-v80-papini18a}\footnote{\url{https://github.com/rll/rllab}}. We test the performance of our algorithms as well as the baseline algorithms on the Cart-Pole\footnote{\url{https://github.com/openai/gym/wiki/CartPole-v0}} environment and the Mountain-Car environment. 

For baseline algorithms, We choose GPOMDP~\citep{baxter2001infinite} and two variance-reduced policy gradient algorithms SVRPG~\citep{pmlr-v80-papini18a} and SRVRPG~\citep{xu2019sample}. The results and detailed experimental design are described as follows:

\subsection{Comparison of different Algorithms}
    
In SRVRPG \citep{xu2019sample} and SVRPG \citep{pmlr-v80-papini18a}, adjustable parameters include the large batch size $S_0$, the mini batch size $B$, the inner iteration number $m$ and the learning rate $\eta$. In \AlgoName Algorithm, we have to tune the large batch size $S_0$, the momentum factor $a$ and the learning rate $\eta$. Notice that we do not tune the mini batch size $B$ in STORM-PG, and fix it to be the same with the best $B$ tuned on SVRPG, as shown in the theory.

We adaptively choose the learning rate by adam optimizer and learning rate decay. The initial learning rate and decay discount are chosen between $(0.0001, 0.1)$ and $(0.5, 0.99)$ respectively. The environment related parameters: the discount factor $\gamma$ and the horizon $H$ varies according to tasks. We list the specific choice of $\gamma$, $H$, together with the  initial batch size $S_0$ and the inner batch size $B$ in the supplementary materials.

We use a Gaussian policy with a neural network with one hidden layer of size 64. For each algorithm in one environment, we choose ten best independent runs to collect the rewards and plot the confidence interval together with the average rewards at each iteration of the training process.

\paragraph{Cart-Pole environment}
The Cart-Pole environment describes the interaction of a pendulum pole attached to a cart. By pushing the cart leftward or rightward, a reward of +1 is obtained by keeping the pole upright and the episode ends when the cart or the pole is too far away from a given center.

Under this environment setting, figure~\ref{fig:label} shows the growth of the average return according to the training trajectories.

\begin{figure}[!tb]
\centering
\includegraphics[scale=0.5]{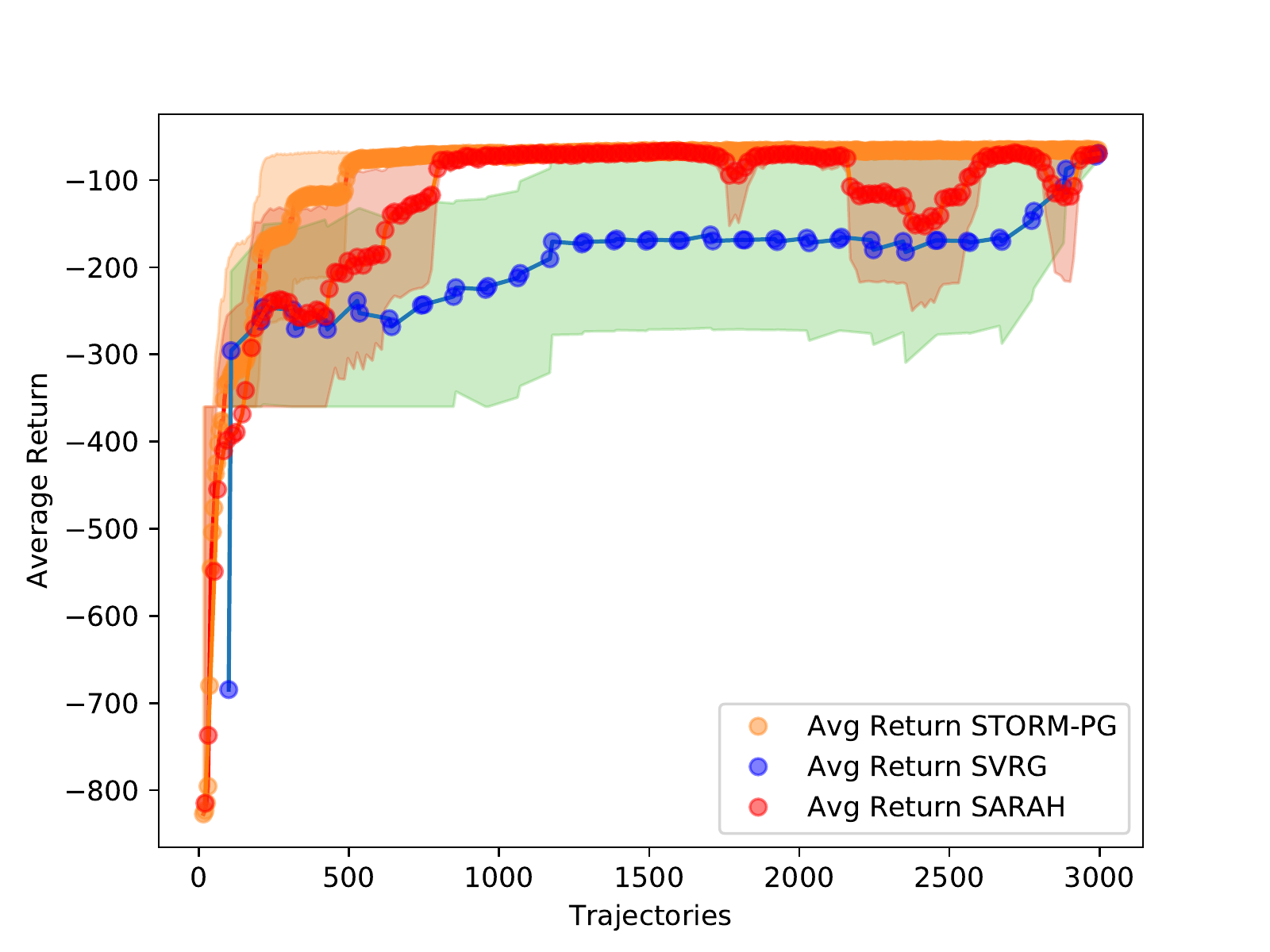}
\caption{A comparison between different policy gradient algorithms on Mountain-Car task. The $x$-axis is the trajectories sampled, the $y$-axis is the average return of the policy parameter.}
\label{fig:label2}
\end{figure}

From Figure~\ref{fig:label}, we see that our STORM-PG algorithm ourperforms other variance-reduced policy gradient methods in convergence speed. It reaches the maximum value at approximately 500 trajectories while SRVRPG and SVRPG reaches the maximum value at approximately 1500 trajectories. GPOMDP converges at about 3000 trajectories.

\paragraph{Mountain-Car Environment}
We use the Mountain Car environment provided in rllab environments. The task is to push a car to a certain position on a hill. The agent takes continuous actions to move leftward or rightward and gets a reward according to it's current position and height, every step it takes gets a penalty -1 and the episode ends when a target position is reached.

Figure~\ref{fig:label2} shows the growth of the average return according to the training trajectories. The GPOMDP algorithm in the Mountain-Car environment does not converge well. For illustrative purpose we only present the plot of the STORM-PG algorithm and two variance-reduced baselines.

From Figure~\ref{fig:label2}, we see that STORM-PG algorithm outperforms other baselines within the first 200 trajectories and reaches a stable zone within 600 trajectories, while for the algorithms it takes at least 1000 trajectories to reach a reasonable result. The two figures~\ref{fig:label} and~\ref{fig:label2} verifies our theory that our \AlgoName algorithm brings significant improvement to the policy gradient training.

Specifically, as we have mentioned at the beginning of Section~\ref{sec:exp}, previous variance-reduced policy gradient methods requires carefully tuning of the inner loop iteration number. SVRPG~\citep{pmlr-v80-papini18a} uses adaptive number of iterations while after tuning SRVRPG~\citep{xu2019sample} fixes a very small number of inner loops. 

On the contrary, we do not tune the mini batch size $B$. In practice, we fix both the initial batch $S_0$ and the mini batch $B$. The high stability with respect to hyper-parameters saves lots of efforts during the training process, The tolerance to the choice of parameters allows us to design a highly user friendly while efficient policy gradient algorithm. 

\section{Final Remarks}
\label{sec:final}
In this paper, we propose a new \AlgoName algorithm that adopts a recently proposed variance-reduced gradient method called STORM.
\AlgoName enjoys advantage both theoretically and experimentally.
From the final experimental results, our STORM-PG algorithm is significantly better than all other baseline methods, both in aspects of training stability and parameter tuning (the user time of tuning STORM-PG is much shorter).
The superiority of STORM-PG in experimental results over SVRPG breaks the curse that stochastic recursive gradient method, namely SARAH, often fails to outperform SVRG in practice even though it has better theoretical convergence rate.
Future works include proving the lower bounds of our algorithm and further improvement of the experimental performance on other statistical learning tasks.
We hope this work can inspire both reinforcement learning and optimization communities for future explorations.

\nocite{langley00}

\bibliography{merged}
\bibliographystyle{icml2020}
\appendix
\onecolumn

\section{Proof of Auxillary Lemmas}
\subsection{Proof of Lemma~\ref{lem:recursive}}
\begin{proof}[Proof of Lemma~\ref{lem:recursive}]
$$\begin{aligned}
 &\quad
\discount \sum_{t=0}^{T-1} \EE\|\bg_t - \nabla_{\bxi} L(\bxi_t)\|^2
\\&=
\sum_{t=0}^{T-1} \EE\|\bg_t - \nabla_{\bxi} L(\bxi_t)\|^2
 - 
(1-\discount) \sum_{t=0}^{T-1} \EE\|\bg_t - \nabla_{\bxi} L(\bxi_t)\|^2
\\&=
\sum_{t=1}^{T} \EE\|\bg_t - \nabla_{\bxi} L(\bxi_t)\|^2
 - 
(1-\discount) \sum_{t=0}^{T-1} \EE\|\bg_t - \nabla_{\bxi} L(\bxi_t)\|^2
 - \EE\left[ \|\bg_T - \nabla_{\bxi} L(\bxi_T)\|^2  - \|\bg_0 - \nabla_{\bxi} L(\bxi_0)\|^2\right]
 \\&\overset{(a)}{\le}
\sum_{t=1}^{T} \EE\|\bg_t - \nabla_{\bxi} L(\bxi_t)\|^2
- 
(1-\discount)^2 \sum_{t=0}^{T-1} \EE\|\bg_t - \nabla_{\bxi} L(\bxi_t)\|^2
-
\EE\left[ \|\bg_T - \nabla_{\bxi} L(\bxi_T)\|^2  - \|\bg_0 - \nabla_{\bxi} L(\bxi_0)\|^2\right]
 \\&\overset{(b)}{\le}
\frac{2\eta^2}{B}(1-\discount)^2 C_\gamma^2 \sum_{t=0}^{T-1} \EE\|\bg_t\|^2
+
\frac{2T \discount^2 \sigma^2}{B}
+
\EE\left[\|\bg_0 - \nabla_{\bxi} L(\bxi_0)\|^2\right]
 \\&\le
\frac{2C_\gamma^2 \eta^2}{B} \sum_{t=0}^{T-1} \EE\|\bg_t\|^2
+
\frac{2T \discount^2 \sigma^2}{B}
+
2\EE\left[\|\bg_0 - \nabla_{\bxi} L(\bxi_0)\|^2\right]
.
\end{aligned}$$

In $(a)$ we used $\discount = (1-\discount) \ge (1-\discount)^2$ as $a \in [0, 1]$. $(b)$ is a direct result from Equation~\eqref{eq:one_recursive} in Lemma~\ref{lem:one_recursive}

\end{proof}

\subsection{Proof of Lemma~\ref{lem:one_recursive}}
\begin{proof}[Proof of Lemma~\ref{lem:one_recursive}]

In \AlgoName, as $\EE\|\bg_{t+1} - \nabla_{\bxi} L(\bxi_{t+1})\|^2 = \EE[\EE[\|\bg_{t+1} - \nabla_{\bxi} L(\bxi_{t+1})\|^2\mid \cF_t]] $, we could first take the conditional expectation of $\bg_{t+1} - \nabla_{\bxi} L(\bxi_{t+1})$over $\cF_t$, where $\cF_t$ is defined as the information before time $t$.

\begin{equation}\begin{aligned}
&\quad
\EE \left[\|\bg_{t+1} - \nabla_{\bxi} L(\bxi_{t+1})\|^2 \mid \cF_t\right]
\\&= 
\EE\left[\left\|(1 - \discount)\left(\bg_t + \frac{1}{B}\sum_{i\in \cB}\left[d_i(\bxi_{t+1}) - d_i^{\bxi_{t+1}}(\bxi_{t})\right]\right)\right.\right.
\left.\left.
+ \frac{1}{B}\sum_{i\in\cB}  \discount d_i(\bxi_{t+1}) - \nabla_{\bxi} L(\bxi_{t+1})\right\|^2\mid \cF_t\right]
\\&=
\EE [\|(1-\discount) (\bg_t - \nabla_{\bxi} L(\bxi_t))
+
 (1-\discount)\bigg[
\frac{1}{B}\sum_{i\in\cB}\left[d_i(\bxi_{t+1}) - d_i^{\bxi_{t+1}}(\bxi_{t})\right]
- \nabla_{\bxi} L(\bxi_{t+1}) + \nabla_{\bxi} L(\bxi_t)
\bigg] 
\\&\quad
+ \discount [\frac{1}{B}\sum_{i\in\cB}d_i(\bxi_{t+1}) -\nabla_{\bxi} L(\bxi_{t+1})]\|^2\mid \cF_t]
\\&=
(1-\discount)^2 \|\bg_t - \nabla_{\bxi} L(\bxi_t)\|^2 
+\frac{1}{B} 2\discount^2 \sigma^2
+\frac{1}{B} 2(1-\discount)^2 \EE\bigg[
\| d_i(\bxi_{t+1}) - d_i^{\bxi_{t+1}}(\bxi_{t})
 - \nabla_{\bxi} L(\bxi_{t+1}) + \nabla_{\bxi} L(\bxi_t) \|^2\mid\cF_t
\bigg]
\\&\overset{(a)}{\le}
(1-\discount)^2 \|\bg_t - \nabla_{\bxi} L(\bxi_t)\|^2 
+\frac{1}{B} 2\discount^2 \sigma^2
+ \frac{1}{B}2(1-\discount)^2 \EE\bigg[
\left\| d_i(\bxi_{t+1}) - d_i^{\bxi_{t+1}}(\bxi_{t})  \right\|^2\mid\cF_t
\bigg]
\\&\overset{(b)}{\le}
(1-\discount)^2 \|\bg_t - \nabla_{\bxi} L(\bxi_t)\|^2
+\frac{1}{B}2 \discount^2 \sigma^2
+\frac{1}{B} 2(1-\discount)^2 C_\gamma^2\EE\left[\|\bxi_{t+1} - \bxi_t\|^2\mid \cF_t\right]
\\&\overset{(c)}{\le}
(1-\discount)^2 \|\bg_t - \nabla_{\bxi} L(\bxi_t)\|^2 
+\frac{1}{B} 2\discount^2 \sigma^2
+ \frac{1}{B}2\eta^2(1-\discount)^2 C_\gamma^2\EE\|\bg_t\|^2
.
\end{aligned}\end{equation}

In the derivations above, $(a)$ comes from the fact that for a dummy random vector $Z$, $\EE\|Z - \EE(Z)\|^2 \le \EE \|Z\|^2$. $(b)$ is due to the $L$-smoothness of $f_i$, and $(c)$ is a direct result of the update rule of $\bxi_t$.

Taking expectation over all randomness in each iterations, we have
$$
\begin{aligned}
&\quad\EE\left[\|\bg_{t+1} - \nabla_{\bxi} L(\bxi_{t+1})\|^2\right]
\\&=
\EE\left[\EE\left[\|\bg_{t+1} - \nabla_{\bxi} L(\bxi_{t+1})\|^2\mid \cF_t\right]\right]
\\&\le
\bigg[
(1-\discount)^2 \EE\|\bg_t\| - \nabla_{\bxi} L(\bxi_t)\|^2
+ \frac{2\discount^2 \sigma^2}{B}
 + \frac{2\eta^2}{B}(1-\discount)^2 C_\gamma^2\EE\|\bg_t\|^2
\bigg]
\\&=
\bigg[
(1-\discount)^2 \EE\|\bg_t - \nabla_{\bxi} L(\bxi_t)\|^2
+ \frac{2\discount^2 \sigma^2}{B}
 +\frac{2 \eta^2}{B}(1-\discount)^2 C_\gamma^2\EE\|\bg_t\|^2
\bigg]
,
\end{aligned}
$$
which completes our proof.
\end{proof}

\subsection{Proof of Lemma~\ref{lem:diff}}
\label{sec:append:lem:diff}
\begin{proof}[Proof of Lemma~\ref{lem:diff}]
\begin{equation}
\begin{aligned}
&\quad \EE\|d_i^{\bxi_{t+1}}(\bxi_t) - d_i(\bxi_{t+1})\|^2
\\&\le
2\EE\|d_i^{\bxi_{t+1}}(\bxi_t) - d_i(\bxi_t)\|^2 + 2\EE\|d_i(\bxi_t) - d_i(\bxi_{t+1})\|^2
\\&=
2\EE\|\sum_{h=0}^{H-1}(w_h(\tau\mid \bxi_t, \bxi_{t+1}) - 1)d_{i,h}\|^2 +
2\EE\|d_i(\bxi_t) - d_i(\bxi_{t+1})\|^2
\\&=
2\EE\|\sum_{h=0}^{H-1}(w_h(\tau\mid \bxi_t, \bxi_{t+1}) - 1)\left(\sum_{t=0}^h \nabla \log \pi_{\bxi}(a_t \mid s_t) \right)\gamma^h r(s_h, a_h)\|^2 +
2\EE\|d_i(\bxi_t) - d_i(\bxi_{t+1})\|^2
\\&\le
\sum_{h=0}^{H-1}2\EE\|(w_h(\tau\mid \bxi_t, \bxi_{t+1}) - 1)\|^2 h^2M^2R^2\gamma^{2h}+
2L_d \EE\|\bxi_t - \bxi_{t+1}\|^2
\\&\le
\sum_{h=0}^{H-1}2C_w^2 h^2M^2R^2\gamma^{2h}\EE\|\bxi_t - \bxi_{t+1}\|^2 +
2L_d \EE\|\bxi_t - \bxi_{t+1}\|^2
\\&\le
C\EE\|\bxi_t - \bxi_{t+1}\|^2
\end{aligned}
\end{equation}

\end{proof}

\section{Hyperparameters}
\subsection{Cart-Pole}

\begin{table*}[h]
\centering
\begin{tabular}{c|c|c|c|c|c|c|c}
\hline
 & $\gamma$ & $\eta$ & $B$ & $m$ &$S_0$ & $a$ & $H$\\
\hline
GPOMDP \citep{baxter2001infinite} & 0.99&0.005 &N/A &N/A &25 &N/A &100 \\
SVRPG \citep{xu2019improved} & 0.99 &0.0075 &10 &3 &25 &N/A &100 \\
SRVRPG \citep{xu2019sample} & 0.99&0.005 &5 & 3&25 &N/A &100\\
STORM-PG (This paper) & 0.99 &0.01&5 &N/A &10 &0.9 &100 \\
\hline
\end{tabular}
\caption{Parameters used in running the Cart-Pole experiments.}
\label{tab:tab1}
\end{table*}
\subsection{Mountain-Car}
\begin{table*}[h]
\centering
\begin{tabular}{c|c|c|c|c|c|c|c}
\hline
 & $\gamma$ & $\eta$ & $B$ & $m$ &$S_0$ & $a$ & $H$\\
\hline
SVRPG \citep{xu2019improved} & 0.99 &0.028 & 8&2 &91 &N/A &1000 \\
SRVRPG \citep{xu2019sample} & 0.99&0.018 &9 & 2&11 &N/A &1000\\
STORM-PG (This paper) & 0.99 &0.01&5 &N/A &10 & 0.79& 1000\\
\hline
\end{tabular}
\caption{Parameters used in running the Mountain-Car environments.}
\label{tab:tab2}
\end{table*}
%

%
%
%

\end{document}